\documentclass{article}

\usepackage{microtype}
\usepackage{graphicx}
\usepackage{subfigure}
\usepackage{booktabs} % for professional tables

\usepackage{hyperref}
\usepackage{multirow}

\usepackage[accepted]{icml2025}
\usepackage{amsmath}
\usepackage{amssymb}
\usepackage{mathtools}
\usepackage{amsthm}

\usepackage[capitalize,noabbrev]{cleveref}

\theoremstyle{plain}
\newtheorem{theorem}{Theorem}[section]

\theoremstyle{definition}

\theoremstyle{remark}

\def\sm{\mathrm{softmax}}
\usepackage[disable,textsize=tiny]{todonotes}
\icmltitlerunning{Nexus: Higher-Order Attention Mechanisms in Transformers}

\begin{document}

\twocolumn[
\icmltitle{Nexus: Higher-Order Attention Mechanisms in Transformers}

% It is OKAY to include author information, even for blind
% submissions: the style file will automatically remove it for you
% unless you've provided the [accepted] option to the icml2025
% package.

% List of affiliations: The first argument should be a (short)
% identifier you will use later to specify author affiliations
% Academic affiliations should list Department, University, City, Region, Country
% Industry affiliations should list Company, City, Region, Country

% You can specify symbols, otherwise they are numbered in order.
% Ideally, you should not use this facility. Affiliations will be numbered
% in order of appearance and this is the preferred way.
%\icmlsetsymbol{equal}{*}

% 如果没有共同一作，可以注释掉下面这行
%\icmlsetsymbol{equal}{*} 

\begin{icmlauthorlist}
\icmlauthor{Hanting Chen}{noah}
\icmlauthor{Chong Zhu}{noah}
\icmlauthor{Kai Han}{noah}
\icmlauthor{Yuchuan Tian}{noah}
\icmlauthor{Yuchen Liang}{noah}
\icmlauthor{Tianyu Guo}{noah}
\icmlauthor{Xinghao Chen}{noah}
\icmlauthor{Dacheng Tao}{ntu}
\icmlauthor{Yunhe Wang}{noah}
\end{icmlauthorlist}

\icmlaffiliation{ntu}{Nanyang Technological University}
\icmlaffiliation{noah}{Huawei Noah's Ark Lab}

\icmlcorrespondingauthor{Hanting Chen}{chenhanting@huawei.com}
\icmlcorrespondingauthor{Yunhe Wang}{yunhe.wang@huawei.com}

% You may provide any keywords that you
% find helpful for describing your paper; these are used to populate
% the "keywords" metadata in the PDF but will not be shown in the document
\icmlkeywords{Machine Learning, ICML}

\vskip 0.3in
]

% this must go after the closing bracket ] following \twocolumn[ ...

% This command actually creates the footnote in the first column
% listing the affiliations and the copyright notice.
% The command takes one argument, which is text to display at the start of the footnote.
% The \icmlEqualContribution command is standard text for equal contribution.
% Remove it (just {}) if you do not need this facility.

\printAffiliationsAndNotice{}  % leave blank if no need to mention equal contribution
%\printAffiliationsAndNotice{\icmlEqualContribution} % otherwise use the standard text.

\begin{abstract}

Transformers have achieved significant success across various domains, relying on self-attention to capture dependencies. However, the standard first-order attention mechanism is often limited by a low-rank bottleneck, struggling to capture intricate, multi-hop relationships within a single layer. In this paper, we propose the \textbf{Nexus}, a novel architecture designed to enhance representational power through a recursive framework. Unlike standard approaches that use static linear projections for Queries and Keys, Nexus dynamically refines these representations via nested self-attention mechanisms. Specifically, the Query and Key vectors are themselves outputs of inner attention loops, allowing tokens to aggregate global context and model high-order correlations \textit{prior} to the final attention computation. We enforce a parameter-efficient weight-sharing strategy across recursive steps, ensuring that this enhanced expressivity incurs $\mathcal{O}(1)$ additional parameters. We provide theoretical analysis demonstrating that our method breaks the linear bottleneck of standard attention. Empirically, Nexus outperforms standard Transformers on multiple benchmarks. 
\end{abstract}

\section{Introduction}
The Transformer architecture~\cite{vaswani2017attention} has emerged as a foundational technology for sequence modeling, achieving widespread success in a variety of domains such as natural language processing (NLP)~\cite{brown2020language} and computer vision (CV)~\cite{dosovitskiy2020image}. Its core self-attention mechanism allows for the efficient capture of long-range dependencies, a crucial feature for tasks that require modeling relationships between distant elements in sequences. The architecture's ability to handle large datasets and complex relationships has made it the model of choice for many state-of-the-art solutions across different fields~\cite{openai2023chatgpt,jumper2021highly,bi2023accurate}.

Despite its successes, the standard Transformer faces inherent limitations in expressivity. Recent theoretical works suggest that the self-attention matrix suffers from a rank collapse issue, limiting its ability to model complex, hierarchical relationships essential for logical reasoning. As tasks demand more intricate dependency modeling, simply scaling model depth and width yields diminishing returns. This limitation is particularly evident in tasks involving multi-step reasoning and symbolic manipulation, where the flat structure of standard attention struggles to maintain logical consistency.

In recent years, several methods have been proposed to improve the Transformer model. One approach focuses on reducing the computational complexity of the self-attention mechanism. For instance, Linear Attention methods, such as Linformer \cite{wang2020linformer} and Performer \cite{choromanski2020rethinking}, aim to approximate the full attention matrix with low-rank projections or kernel-based approximations, thereby reducing the time complexity from $O(n^2)$ to $O(n)$. Similarly, Reformer \cite{kitaev2019reformer} introduces locality-sensitive hashing to approximate the attention matrix, lowering the complexity to $O(n \log n)$. Although these new mechanisms can significantly reduce computational complexity, the model's capacity and ability to capture complex relationships will be severely diminished.

In order to break through the model's capacity limitations, more significant research has focused on enhancing the representational power of Transformers to better handle complex tasks. For example, Attention on Attention \cite{huang2019attention} introduces additional layers of attention mechanisms to boost the model's ability to capture intricate dependencies. Similarly, Deformable Attention \cite{xia2022vision} draws on the idea of deformable convolutions, enabling the attention mechanism to adapt more flexibly to the diverse requirements of different tasks. This method demonstrates stronger performance when dealing with irregular inputs. \citet{wei2023multimodal} have proposed high-order relation mechanisms, such as introducing multi-level structures into the self-attention mechanism to enhance the model's ability to represent multimodal data. These approaches have further improved the ability of Transformers to process complex hierarchical structures. These advancements highlight the potential of Transformers to handle richer, multi-level data representations, which enhances performance on tasks that require modeling more complex, hierarchical dependencies. However, despite these advancements, such improvements have primarily been applied to vision models, and have not been as extensively explored in language models due to the unique characteristics of natural language tasks.

To bridge this gap, we propose a novel \textbf{Higher-Order Attention Mechanism} specifically designed to enhance the reasoning capabilities of Transformers. Instead of simple linear projections, our approach recursively refines Query and Key vectors through nested attention loops. This recursive formulation allows the model to capture high-order interactions—effectively performing a "pre-reasoning" step to align semantic relationships before the final attention computation. While recursive computation inherently increases inference FLOPs, we introduce a \textbf{Weight-Shared} strategy that maintains the same parameter count as standard Transformers. This design presents a strategic trade-off: investing more compute per token to achieve significantly higher reasoning density and expressivity. We demonstrate both theoretically and empirically that this mechanism breaks the low-rank bottleneck. Furthermore, we show that Nexus is not just a pre-training architecture but a powerful "upgrade kit" for existing LLMs, significantly boosting the mathematical reasoning abilities of models like Qwen2.5 through architectural retrofitting.

\begin{figure*}[ht]
	\centering
	\includegraphics[width=0.9\linewidth]{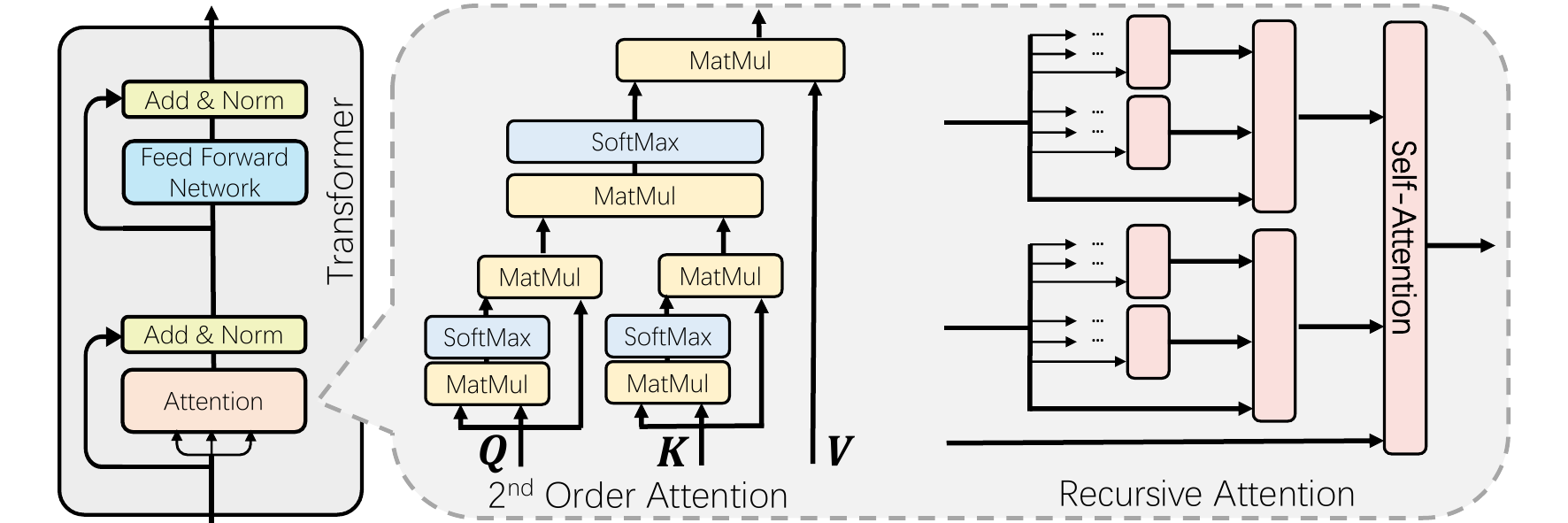}
	\caption{\textbf{Overview of the Nexus.} 
    The figure illustrates the hierarchical structure of our proposed mechanism. 
    \textbf{Left:} The integration of the Nexus layer within a standard Transformer block, replacing the conventional self-attention module. 
    \textbf{Middle:} The detailed architecture of the \textbf{2nd-Order Attention} mechanism. Unlike standard attention where $Q$ and $K$ are linear projections, Nexus recursively refines $Q$ and $K$ through inner self-attention loops (MatMul $\rightarrow$ SoftMax $\rightarrow$ MatMul) before the final attention computation. This allows the model to capture intricate dependencies prior to the main interaction. 
    \textbf{Right:} The generalized \textbf{Recursive Attention} framework, demonstrating how the mechanism can be extended to arbitrary orders ($m$-th order) to model deeper hierarchical relationships.}
	\label{fig:step_attn}
\end{figure*}

\section{Related Works}

\subsection{Efficient Transformer Models}

While Transformer models~\cite{vaswani2017attention} have demonstrated state-of-the-art performance~\cite{openai2023chatgpt,touvron2023llama}, they suffer from quadratic time complexity, $O(n^2)$, with respect to sequence length $n$. This computational bottleneck becomes particularly problematic in tasks involving long sequences. To address this challenge, several strategies have been explored to reduce the computational cost of self-attention mechanisms. One prominent approach is to approximate the full attention matrix with low-rank projections. For instance, Linformer \cite{wang2020linformer} leverages low-rank matrix factorization to reduce the complexity of self-attention from $O(n^2)$ to $O(n)$, assuming that the attention matrix can be approximated by a low-dimensional subspace. Another approach, Performer \cite{choromanski2020rethinking}, replaces the traditional Softmax attention with a kernel-based approximation, which enables the model to compute attention with $O(n)$ complexity, using a probabilistic approach based on normal distribution. Furthermore, Reformer \cite{kitaev2019reformer} employs locality-sensitive hashing to approximate the attention matrix, reducing the complexity to $O(n \log n)$ by limiting attention computation to hash buckets of similar tokens. Other methods, such as Sparse Transformer \cite{child2019generating} and Longformer \cite{beltagy2020longformer}, sparsify the attention mechanism by limiting the interactions between sequence positions, thus enabling more efficient computation while maintaining model performance on long sequences. Despite their efficiency, these models often struggle to fully retain the expressiveness and fine-grained relationships captured by the original $O(n^2)$ self-attention mechanism, especially in highly complex tasks.

\subsection{Enhancing Representation Power of Transformers}

While reducing the computational complexity of Transformers has been a major focus, another important line of research aims to enhance the representational capacity of Transformers to handle increasingly complex tasks. As Transformer models are applied to more sophisticated NLP problems, such as multi-modal understanding and hierarchical reasoning, their ability to represent nuanced and multi-level relationships becomes critical. A simple extension of the traditional attention mechanism is often not sufficient to capture the intricate structures present in such tasks.

To overcome this limitation, various methods have been proposed to increase the expressive power of Transformers. Attention on Attention \cite{huang2019attention}, for example, introduces additional layers of attention to allow the model to capture more complex dependencies. This approach adds a second level of attention, which refines the relationships between tokens based on higher-order interactions. Similarly, Deformable Attention \cite{xia2022vision} incorporates the idea of deformable convolutions into the attention mechanism. By allowing the attention operation to adapt to varying task demands, it provides the model with a more flexible and task-specific attention mechanism, particularly useful when dealing with irregular or sparse inputs.

Another promising approach is the introduction of high-order relation extraction~\cite{wei2023multimodal}, where additional layers or recursive attention operations are used to enhance the model's ability to process multi-level, hierarchical data. These mechanisms enable the Transformer to handle more complex, multi-modal tasks and capture long-range dependencies more effectively. Higher-Order Attention Network~\cite{hajij2022higher} introduces a new class of graph neural networks built on combinatorial complexes, which generalize both hypergraphs and cell complexes, demonstrating competitive or superior performance on tasks like mesh shape analysis and graph learning.

However, while these methods significantly enhance the representational power of Transformers, they are often applied primarily in vision or graph tasks. The adoption of these advanced mechanisms in NLP tasks remains less explored, as natural language has unique challenges that differ from those encountered in vision, such as syntactic and semantic complexities.

\subsection{Reasoning and Chain-of-Thought in LLMs}
The ability of LLMs to perform complex reasoning, often elicited via Chain-of-Thought (CoT) prompting \cite{wei2022chain}, has become a focal point of research. While scaling laws \cite{kaplan2020scaling} suggest that larger models generally reason better, architectural limitations remain. Standard self-attention treats tokens as a flat sequence, which may struggle to maintain the rigid logical consistency required for multi-step mathematical or symbolic reasoning \cite{dziri2023faith}. Recent works have attempted to bolster reasoning via external verifiers \cite{cobbe2021training} or iterative prompting strategies. In contrast, our work seeks to internalize this iterative reasoning capability directly into the attention mechanism. By recursively refining Queries and Keys, the Nexus effectively performs a coherent "pre-reasoning" step within the layer itself, potentially reducing the reliance on massive depth for logical deduction.

\section{Method}

In this section, we introduce the proposed Higher-Order Attention Mechanism, designed to enhance the Transformer architecture's ability to capture complex, hierarchical dependencies and support chain-of-thought (CoT) reasoning. We begin by analyzing the limitations of the standard self-attention mechanism, followed by the definition of our Higher-Order Attention. We then extend this mechanism recursively to capture multi-level dependencies and integrate it into the Transformer framework. Finally, we provide a theoretical analysis of the advantages of our approach.

\subsection{Challenges of the Standard Self-Attention Mechanism}

The standard self-attention mechanism in Transformers computes the attention output as follows:
\begin{equation}
	\text{Attention}(X) = \text{softmax}\left(\frac{QK^\top}{\sqrt{d_k}}\right)V,
\end{equation}
where \( Q = W_q X \), \( K = W_k X \), and \( V = W_v X \) are the transformed query, key, and value matrices, respectively, and \( d_k \) is the dimensionality of the key vectors. Here, \( X \in \mathbb{R}^{n \times d} \) represents the input sequence with length \( n \) and feature dimension \( d \).

While the self-attention mechanism is effective at capturing pairwise dependencies between elements in the sequence, it inherently limits the model's ability to represent higher-order interactions. Specifically, the attention weights \( A = \text{softmax}\left(\frac{QK^\top}{\sqrt{d_k}}\right) \in \mathbb{R}^{n \times n} \) only model direct interactions between pairs of tokens. This pairwise focus hinders the model's capacity to perform multi-step reasoning or to capture more intricate, hierarchical relationships that span multiple tokens. As a result, standard Transformer models encounter performance bottlenecks when dealing with tasks that require modeling complex, multi-level dependencies or long-range contextual information. This limitation underscores the need for enhanced attention mechanisms that can support deeper interactions and more sophisticated reasoning processes within the Transformer architecture.

To comprehensively understand the constraints of the standard self-attention mechanism, we delved into its interaction dynamics and expressive capabilities. The conventional self-attention computes pairwise interactions between tokens, effectively capturing first-order dependencies. Mathematically, for any two tokens \( x_i \) and \( x_j \) in the input sequence \( X \), the attention output focuses solely on their direct interaction as expressed by the attention weights \( A_{ij} \):
\begin{equation}
	A_{ij} = \text{softmax}\left(\frac{Q_i K_j^\top}{\sqrt{d_k}}\right)
\end{equation}
where \( Q_i = W_q x_i \) and \( K_j = W_k x_j \) are the query and key vectors for tokens \( x_i \) and \( x_j \), respectively, and \( d_k \) is the dimensionality of the key vectors. This formulation inherently restricts the model to capturing only pairwise (first-order) interactions between tokens.

However, many complex tasks require understanding higher-order dependencies, where the relationship among three or more tokens is essential. For instance, in natural language processing tasks such as semantic role labeling or multi-hop reasoning, the meaning or inference often depends on the interplay among multiple tokens simultaneously. To model such interactions, higher-order dependencies need to be explicitly represented. 

Consider the interaction among three tokens \( x_i \), \( x_j \), and \( x_k \). The standard self-attention mechanism would require sequential pairwise computations to infer the combined effect of these tokens:
\begin{align}
	A_{ij} &= \text{softmax}\left(\frac{Q_i K_j^\top}{\sqrt{d_k}}\right), \\
	A_{ik} &= \text{softmax}\left(\frac{Q_i K_k^\top}{\sqrt{d_k}}\right), \\
	A_{jk} &= \text{softmax}\left(\frac{Q_j K_k^\top}{\sqrt{d_k}}\right).
\end{align}
To capture the triadic interaction among \( x_i \), \( x_j \), and \( x_k \), multiple layers of self-attention would need to be stacked, or iterative reasoning steps would have to be performed. This not only increases the computational burden but also propagates potential information loss and gradient vanishing issues across layers, thereby limiting the model's ability to effectively learn and utilize higher-order relationships.

Moreover, representing higher-order interactions through repeated pairwise computations is inherently inefficient and may not scale well with the increasing complexity of dependencies in longer sequences. The necessity to stack multiple Transformer layers or perform multi-step reasoning introduces additional parameters and computational overhead, which can be detrimental to both training and inference efficiency.

To address these challenges, it is imperative to develop an attention mechanism that can natively capture higher-order interactions among multiple tokens within a single computational step. Such a mechanism would enhance the model's capacity to represent complex, multi-token dependencies without the need for excessive layering or iterative processes, thereby improving both efficiency and performance in tasks requiring sophisticated reasoning capabilities.

\subsection{Higher-Order Attention Mechanism}

To overcome the aforementioned limitations of the standard self-attention mechanism, we introduce the \textbf{Higher-Order Attention Mechanism}. This mechanism is designed to capture multi-token interactions directly, enabling the model to represent higher-order dependencies within a single attention computation.
\begin{equation}
	\begin{aligned}
		&\text{H-Attention}(X) = \\
		&\text{Attention}(\text{Attention}_q(X), \text{Attention}_k(X), V)
	\end{aligned}
\end{equation}
Here, \( \text{Attention}_q(X) \) and \( \text{Attention}_k(X) \) are refined query and key representations obtained by applying the self-attention mechanism separately to the queries and keys:
\begin{align}
	\text{Attention}_q(X) &= \text{softmax}\left(\frac{Q Q^\top}{\sqrt{d_k}}\right) Q \\
	\text{Attention}_k(X) &= \text{softmax}\left(\frac{K K^\top}{\sqrt{d_k}}\right) K
\end{align}
By first refining the queries and keys through self-attention, the Higher-Order Attention Mechanism effectively incorporates higher-order interactions into the attention computation. Specifically, \( \text{Attention}_q(X) \) and \( \text{Attention}_k(X) \) encapsulate aggregated information from multiple tokens, allowing the subsequent attention operation to consider multi-token dependencies directly.

This can be formalized as:

\begin{equation}
	\small
	\begin{aligned}
		&\text{H-Attention}(X) =\\ &\text{softmax}\left(\frac{\left(\text{Attention}_q(X)\right)\left(\text{Attention}_k(X)\right)^\top}{\sqrt{d_k}}\right)V
	\end{aligned}
\end{equation}

The incorporation of \( \text{Attention}_q(X) \) and \( \text{Attention}_k(X) \) enables the model to capture interactions that span beyond pairwise token relationships. For example, in the case of three tokens \( x_i \), \( x_j \), and \( x_k \), the Higher-Order Attention can simultaneously consider their combined influence on the output, effectively modeling triadic dependencies within a single attention layer.

This approach significantly enhances the model's capacity to perform multi-token reasoning and to capture hierarchical structures inherent in complex data. By embedding higher-order interactions directly into the attention mechanism, the Transformer architecture gains a more expressive power, allowing it to handle tasks that require sophisticated dependency modeling with greater efficiency and reduced computational overhead.

\subsection{Recursive Higher-Order Attention}

Building upon the Higher-Order Attention, we introduce a recursive extension to further enhance the model's ability to capture hierarchical and long-range dependencies. The Recursive Higher-Order Attention is defined as:

\begin{equation}
	\small
	\begin{aligned}
		&\text{H}^2\text{-Attention}(X) = \\
		&\text{Attention}(\text{H-Attention}_q(X), \text{H-Attention}_k(X), V)
	\end{aligned}
\end{equation}

Generalizing this, the \( n \)-th order attention can be defined recursively as:

\begin{equation}
	\small
	\begin{aligned}
		&\text{H}^m\text{-Attention}(X) = \\
		&\text{Attention}(\text{H}^{m-1}\text{-Attention}_q(X), \text{H}^{m-1}\text{-Attention}_k(X), V)
	\end{aligned}
\end{equation}

Each recursive step further processes the query and key vectors, enabling the model to capture multi-level dependencies and more intricate relational structures within the data. This recursive formulation allows the Transformer to effectively perform multi-step reasoning, akin to chain-of-thought processes, by integrating information from successive layers of attention computation.

\subsection{Parameter-Efficient Weight Sharing}
A potential concern with higher-order mechanisms is the increase in parameter count. A naive implementation of $m$-th order attention would require distinct projection matrices for each recursive step, scaling the parameters linearly with $m$. 

To address this, we propose a \textbf{Weight-Shared Higher-Order Attention} strategy. Based on the hypothesis that the semantic transformation required to project a vector into a "Query" or "Key" space is fundamentally similar across recursive levels, we enforce weight sharing between the inner and outer attention layers.

Formally, for a standard attention layer defined by parameters $\theta = \{W_q, W_k, W_v\}$, the recursive step in Equation (8) is modified such that the inner attention mechanisms reuse the same $\theta$:
\begin{equation}
    \text{H-Attention}_q(X; \theta) = \text{Attention}(X, X, X; \theta) \cdot W_q,
\end{equation}
where the inner attention uses the same projection weights as the outer loop.
This constraint ensures that the parameter complexity of Nexus remains identical to that of a standard Transformer, i.e., $\mathcal{O}(1)$ with respect to the recursive order $m$. As demonstrated in our ablation study, this strategy maintains high performance while maximizing parameter efficiency.

\subsection{Complexity Analysis}

Introducing the Higher-Order Attention Mechanism inherently increases the computational complexity of the attention operations within the Transformer architecture. To quantify this increase, we analyze both the standard self-attention mechanism and the proposed Higher-Order Attention Mechanism.

In the standard self-attention mechanism, the computational complexity for this operation is $\mathcal{O}(n^2 d_k)$ with \( n \) being the sequence length and \( d_k \) the dimensionality of the key vectors.

The Higher-Order Attention Mechanism extends the standard self-attention by introducing \( m \)-th order interactions. To determine the computational complexity \( T(m) \) for the \( m \)-th order attention, consider the recursive nature of the mechanism:
\begin{equation}
	T(m) = 2 \cdot T(m-1) + \mathcal{O}(n^2 d_k)
\end{equation}
with the base case \( T(0) = \mathcal{O}(n^2 d_k) \). Thus, the computational complexity of the \( m \)-th order Higher-Order Attention Mechanism is $\mathcal{O}\left(2^{m} n^2 d_k\right)$. 

While the time complexity scales exponentially with the recursive order $m$, it is important to note that in practice, a small recursion depth (e.g., $m=2$) is sufficient to achieve significant performance gains, as shown in our ablation studies. Consequently, the actual computational overhead is a constant factor (approximately $2\times$ for $m=2$) compared to standard attention. 

Crucially, thanks to our \textbf{Weight-Sharing} strategy, the \textit{parameter complexity} remains $\mathcal{O}(1)$ with respect to $m$. This makes Nexus highly memory-efficient during training compared to simply stacking more Transformer layers, as it does not increase the model size (storage), only the computational density.

% \subsection{Theoretical Analysis}
\subsection{Linear Bottleneck of Standard Attention Mechanism}
% @inproceedings{bhojanapalli2020low,
	%   title={Low-rank bottleneck in multi-head attention models},
	%   author={Bhojanapalli, Srinadh and Yun, Chulhee and Rawat, Ankit Singh and Reddi, Sashank and Kumar, Sanjiv},
	%   booktitle={International conference on machine learning},
	%   pages={864--873},
	%   year={2020},
	%   organization={PMLR}}
\citet{bhojanapalli2020low} pointed out that when $d_k < n$, the standard attention mechanism lacks the ability to express arbitrary attention weights $A$. This phenomenon is referred to as the \textit{low-rank bottleneck in attention mechanisms}.

\begin{table*}[t]
	\centering
	\caption{Main Results: Zero-shot accuracy comparison between Pythia baselines and our proposed Nexus across different model scales. The best performance for each scale is highlighted in \textbf{bold}.}
	\label{tab:main_results}
	% 稍微调整表格宽度以适应页面，如果太宽可以用 \resizebox
	\begin{tabular}{ll|cccccc|c}
		\toprule
		\textbf{Scale} & \textbf{Model} & \textbf{ARC-C} & \textbf{ARC-E} & \textbf{Hellaswag} & \textbf{LogiQA} & \textbf{PiQA} & \textbf{SciQ} & \textbf{Avg.} \\
		\midrule
		\multirow{2}{*}{70M} 
		& Pythia & \textbf{0.208} & 0.359 & 0.356 & 0.276 & 0.569 & 0.615 & 0.397 \\
		& \textbf{Nexus (Ours)} & 0.204 & \textbf{0.382} & \textbf{0.358} & \textbf{0.287} & \textbf{0.586} & \textbf{0.685} & \textbf{0.417} \\
		\midrule
		\multirow{2}{*}{160M} 
		& Pythia & 0.200 & 0.385 & 0.380 & 0.260 & 0.600 & 0.686 & 0.419 \\
		& \textbf{Nexus (Ours)} & \textbf{0.211} & \textbf{0.405} & \textbf{0.385} & \textbf{0.285} & \textbf{0.605} & \textbf{0.713} & \textbf{0.434} \\
		\midrule
		\multirow{2}{*}{410M} 
		& Pythia & 0.225 & 0.394 & 0.375 & 0.285 & 0.601 & 0.708 & 0.431 \\
		& \textbf{Nexus (Ours)} & \textbf{0.226} & \textbf{0.415} & \textbf{0.384} & \textbf{0.294} & \textbf{0.608} & \textbf{0.733} & \textbf{0.443} \\
		\midrule
		\multirow{2}{*}{1B} 
		& Pythia & \textbf{0.232} & 0.440 & 0.395 & \textbf{0.296} & 0.625 & 0.758 & 0.458 \\
		& \textbf{Nexus (Ours)} & 0.230 & \textbf{0.455} & \textbf{0.399} & 0.290 & \textbf{0.636} & \textbf{0.777} & \textbf{0.465} \\
		\bottomrule
	\end{tabular}
\end{table*}

With the rise of Chain-of-Thought (CoT) reasoning, long-sequence tasks have become increasingly common, making the condition $d_k < n$ very prevalent. As a compromise, one could focus on approximating \(\log(A)\) with a low-rank approximation. Unfortunately, we find that the standard attention mechanism still lacks the ability to adequately represent low-rank matrices:
\begin{theorem} \label{th:lb}
	(Linear Bottleneck) (1) Given any $N$ different inputs \(X_m \in \mathbb{R}^{n \times d}, m=1, \dots, N\) and the corresponding 
	target row stochastic matrices \(A_m \in \mathbb{R}^{n \times n}\), as long as \(rank(\log(A_m)) \leq d_k\), there always 
	exists two mappings \(Q, K: \mathbb{R}^{n \times d} \rightarrow \mathbb{R}^{n \times d_k}\) such that
	\begin{equation} \label{eq:lb}
		\sm(\frac{Q(X_m)K(X_m)^\top}{\sqrt{d_k}}) = A_m, \quad m=1, \dots, N.	
	\end{equation}
	(2) If $d<n-1$, there exist \(A_m \in \mathbb{R}^{n \times n}\)  that satisfies \(rank(\log(A_m))=1\) but \eqref{eq:lb} still does not hold for all 
	linear transformations $Q(X)=XW_q, K(X)=XW_k$.
\end{theorem}
We provide the complete proof in Appendix \ref{ap:A}. The first proposition in theorem \ref{th:lb} demonstrates that if the mappings \(Q\) and \(K\) with respect to \(X\) are sufficiently flexible, the attention mechanism has the ability to represent \(\log(A_m)\) with rank lower than \(d_k\). 
On the other hand, the second proposition shows that the standard attention mechanism fails to represent even a log-attention-weight matrix with a rank of 1.

This motivates us to design more flexible and nonlinear mappings for \(Q\) and \(K\). A natural choice is the \textit{High-Order Attention Mechanism}, where we define:
\begin{equation}
	Q(X) = \text{Attention}_q(X), \quad K(X) = \text{Attention}_k(X).
\end{equation}
% 这个图可以靠前一点，全是字
As depicted in Figure~\ref{fig:attn_maps}, the attention matrices of the Nexus network exhibit more intricate and interconnected patterns compared to the standard Transformer and other higher-order transformers. These richer patterns indicate that the Nexus network effectively captures multifaceted relationships among tokens, facilitating enhanced reasoning and contextual understanding. Specifically, the Nexus network's attention matrices demonstrate a higher degree of connectivity and diversity in attention weights, reflecting its ability to model complex, hierarchical dependencies more comprehensively.

\section{Experiments}

In this section, we conduct extensive experimental validation of the proposed architecture, encompassing results across language models of varying scales. Additionally, we provide detailed analyses to substantiate the effectiveness of our approach.

\subsection{Evaluation on Different Scales}
Given the challenge of replicating the training processes of most language models, as only their checkpoints are openly available, we opted to validate our method using Pythia~\cite{biderman2023pythia}, a model with a fully public dataset and training procedure, enabling fair comparisons. 

We adhered to the exact training settings employed by Pythia, including learning rates, optimizers, and other hyperparameters, and utilized the Pile dataset. The Pile~\cite{gao2020pile} is an 825 GiB corpus of English text, specifically designed for training large-scale language models. This project is composed of 22 distinct, high-quality subsets, both pre-existing and newly constructed, many of which originate from academic or professional sources. This comprehensive and diverse dataset serves as a robust foundation for developing and fine-tuning language models Our Nexus model was trained with the same setting as pre-trained Pythia model. We evaluated our approach on six public datasets used by Pythia: PIQA~\cite{bisk2020piqa}, Hellaswag~\cite{zellers2019hellaswag}, Sciq~\cite{welbl2017crowdsourcing}, ARC-E, ARC-C~\cite{clark2018think}, and LogiQA~\cite{liu2020logiqa}.

Table~\ref{tab:main_results} summarizes the results. The Nexus network outperforms the standard Transformer baseline in terms of average accuracy across all model scales. Notably, our method achieves substantial gains on tasks requiring multi-step reasoning or long-context integration, such as \textbf{SciQ} (+6\% on 70M) and \textbf{PiQA}. While the baseline Pythia model occasionally performs slightly better on retrieval-heavy tasks like ARC-C at specific scales (e.g., 1B), Nexus demonstrates a more robust capability in handling complex logical dependencies, leading to superior overall performance. This aligns with our hypothesis that higher-order attention is particularly beneficial for tasks where the relationship between tokens cannot be captured by simple pairwise interactions.

\subsection{Ablation Study}
To thoroughly validate the design choices of our Higher-Order Attention Network (Nexus), we conducted a series of ablation studies on the 70M parameter scale. We utilized the same training configuration as the Pythia-70M baseline. Our analysis focuses on three key aspects: (1) the selection of projections for higher-order attention, (2) the impact of parameter sharing on efficiency, and (3) the effect of recursive order depth. The results are summarized in Table~\ref{tab:ablation_results}.

\begin{table*}[t]
	\centering
	\caption{Ablation study of different Higher-Order Attention configurations on the 70M scale model. \textbf{Proj.} denotes which projections utilize the attention mechanism (others use linear). \textbf{Shared} indicates whether inner and outer attention layers share weights. \textbf{Order} denotes the recursion depth. The best performance is highlighted in \textbf{bold}.}
	\resizebox{\linewidth}{!}{
		\begin{tabular}{l|ccc|cccccc|c}
			\toprule
			\textbf{Model ID} & \textbf{Proj.} & \textbf{Shared} & \textbf{Order} & \textbf{ARC-C} & \textbf{ARC-E} & \textbf{Hellaswag} & \textbf{LogiQA} & \textbf{PiQA} & \textbf{SciQ} & \textbf{Avg.} \\
			\midrule
			Baseline (Pythia-70M) & - & - & - & 0.208 & 0.359 & 0.356 & 0.277 & 0.569 & 0.615 & 0.397 \\
			\midrule
			\multicolumn{11}{l}{\textit{Effect of Higher-Order Components}} \\
			Nexus-Q  & Q & No & 2 & 0.198 & 0.374 & 0.345 & 0.286 & 0.579 & 0.621 & 0.400 \\
			Nexus-QK  & Q, K & No & 2 & 0.207 & 0.367 & \textbf{0.359} & 0.281 & 0.577 & 0.663 & 0.409 \\
			Nexus-QKV  & Q, K, V & No & 2 & 0.209 & 0.380 & 0.354 & \textbf{0.293} & 0.579 & 0.636 & 0.409 \\
			\midrule
			\multicolumn{11}{l}{\textit{Parameter Efficiency \& Recursive Depth}} \\
			Nexus-QK-Shared  & Q, K & Yes & 2 & 0.206 & 0.357 & 0.367 & 0.273 & 0.601 & 0.630 & 0.406 \\
			Nexus-Recursive  & Q, K & Yes & 3 & \textbf{0.201} & \textbf{0.391} & \textbf{0.368} & 0.287 & \textbf{0.582} & \textbf{0.659} & \textbf{0.415} \\
			\bottomrule
		\end{tabular}
	}
	\label{tab:ablation_results}
\end{table*}

\paragraph{Selection of Higher-Order Components.}
We first investigated which components of the attention mechanism ($Q$, $K$, $V$) benefit most from higher-order processing. As shown in Table~\ref{tab:ablation_results}, applying higher-order attention only to the Query projection (Nexus-Q) yields a marginal improvement over the baseline (0.400 vs. 0.397). However, extending the mechanism to both Query and Key (Nexus-QK) results in a significant performance boost, raising the average accuracy to 0.409. Interestingly, further applying the mechanism to the Value projection (Nexus-QKV) does not provide additional gains compared to Nexus-QK (both achieve $\sim$0.409). This suggests that the core benefit of our method lies in refining the correlation alignment between Queries and Keys, while a simple linear projection suffices for Values. Consequently, we adopted the Q and K configuration for subsequent experiments.

\paragraph{Parameter Efficiency via Weight Sharing.}
While Nexus-QK achieves superior performance, the introduction of separate attention layers for projections increases the parameter count. To improve parameter efficiency, we proposed a weight-sharing strategy (Nexus-QK-Shared), where the parameters of the inner attention layers are shared with the outer main attention layer. Comparing Nexus-QK and Nexus-QK-Shared, we observe a slight performance drop (0.409 $\rightarrow$ 0.406), yet the shared model still consistently outperforms the baseline. This trade-off allows us to maintain the high expressiveness of higher-order attention while keeping the parameter count comparable to the standard Transformer, maximizing parameter utilization.

\paragraph{Impact of Recursive Order.}
Finally, we explored the potential of increasing the recursive depth. By extending the shared model to a 3rd-order attention (Nexus-Recursive), the performance further improves to 0.415, surpassing all other configurations. This result empirically validates our theoretical claim that higher recursive orders can capture more intricate dependencies. However, increasing the order also incurs higher computational costs during inference. Considering the balance between efficiency and accuracy, we selected the 2nd-order shared configuration (Nexus-QK-Shared) as the default setting for our main experiments, as it offers a robust improvement over the baseline with optimal resource efficiency.

\subsection{Visualization of Attention Patterns}
\label{sec:viz}

To intuitively understand how the Higher-Order Attention mechanism processes information differently from the standard Transformer, we visualize the average attention weights across all layers and heads for the 70M scale models. We compare the baseline Pythia model with our proposed Nexus model. For the Nexus model, we examine three distinct attention maps: the distinct inner recursive attentions for Queries ($\text{Nexus}_Q$) and Keys ($\text{Nexus}_K$), and the final outer attention ($\text{Nexus}_{\text{outer}}$). The visualization results are presented in Figure~\ref{fig:attn_maps}.

\begin{figure*}[t]
    \centering
    \begin{minipage}{0.24\linewidth}
        \centering
        \includegraphics[width=\linewidth]{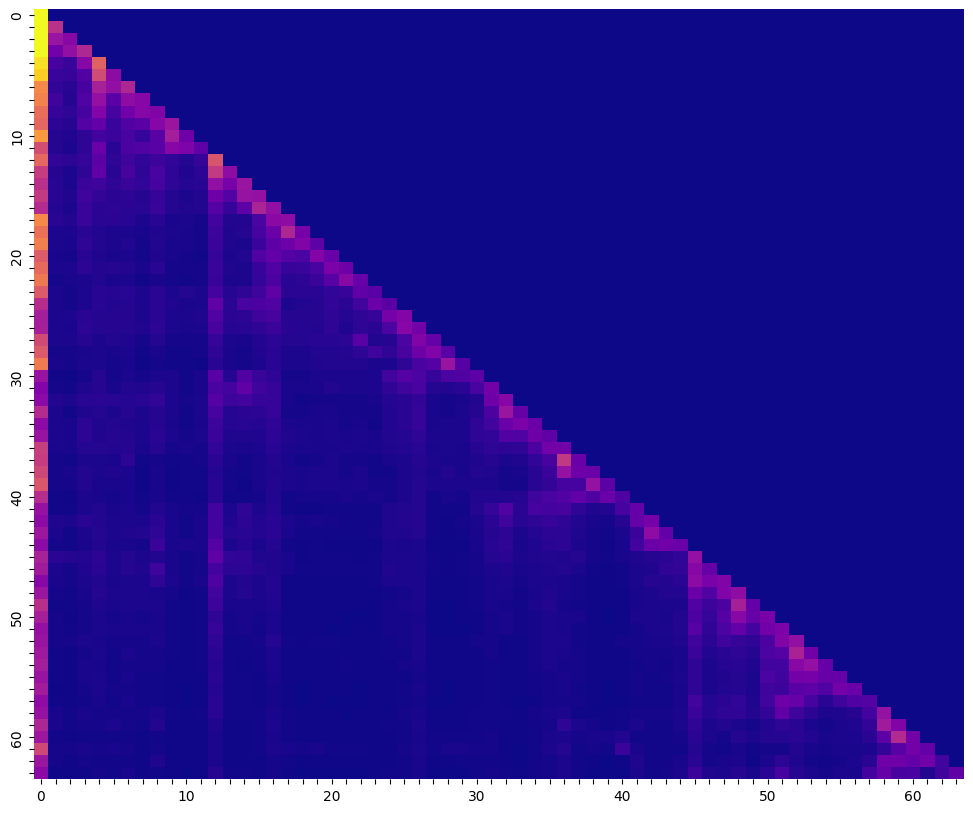}
        \centerline{\small (a) Pythia Baseline}
    \end{minipage}
    \hfill
    \begin{minipage}{0.24\linewidth}
        \centering
        \includegraphics[width=\linewidth]{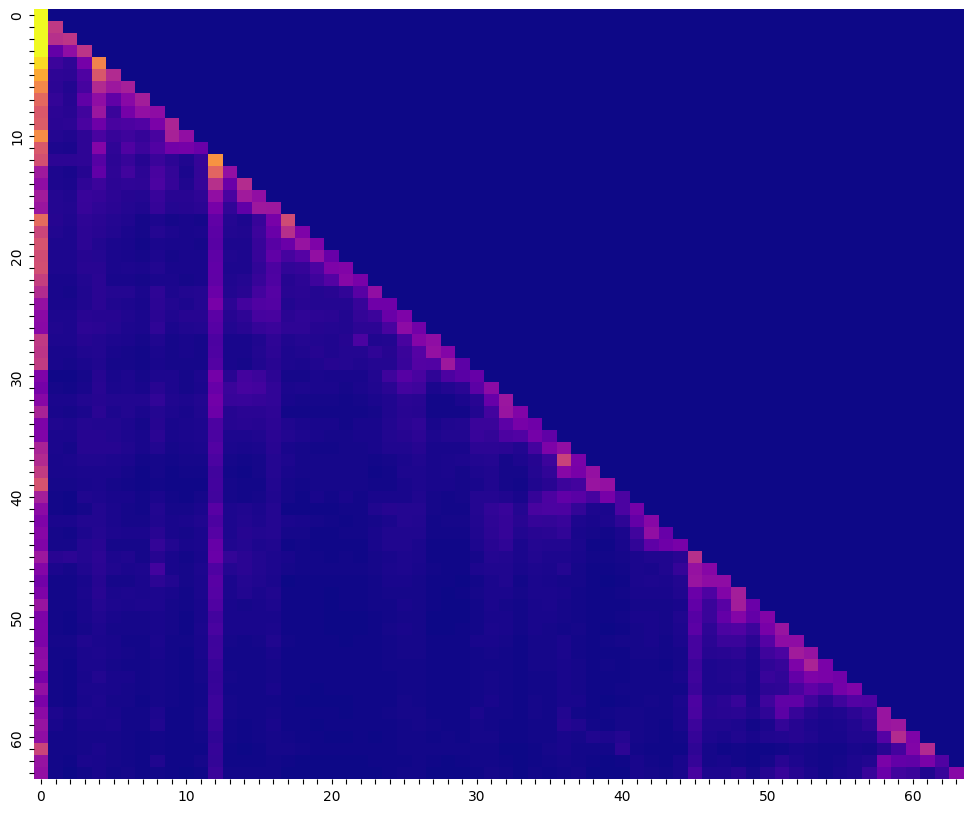}
        \centerline{\small (b) Nexus Outer}
    \end{minipage}
    \hfill
    \begin{minipage}{0.24\linewidth}
        \centering
        \includegraphics[width=\linewidth]{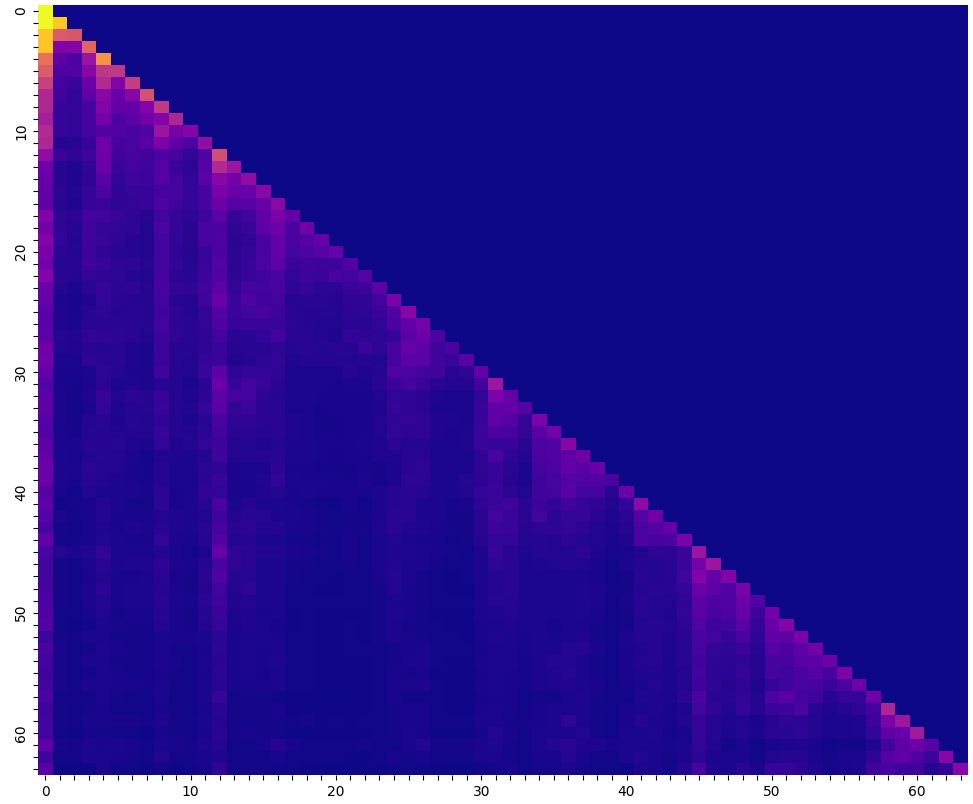}
        \centerline{\small (c) Nexus Inner ($Q$-Attn)}
    \end{minipage}
    \hfill
    \begin{minipage}{0.24\linewidth}
        \centering
        \includegraphics[width=\linewidth]{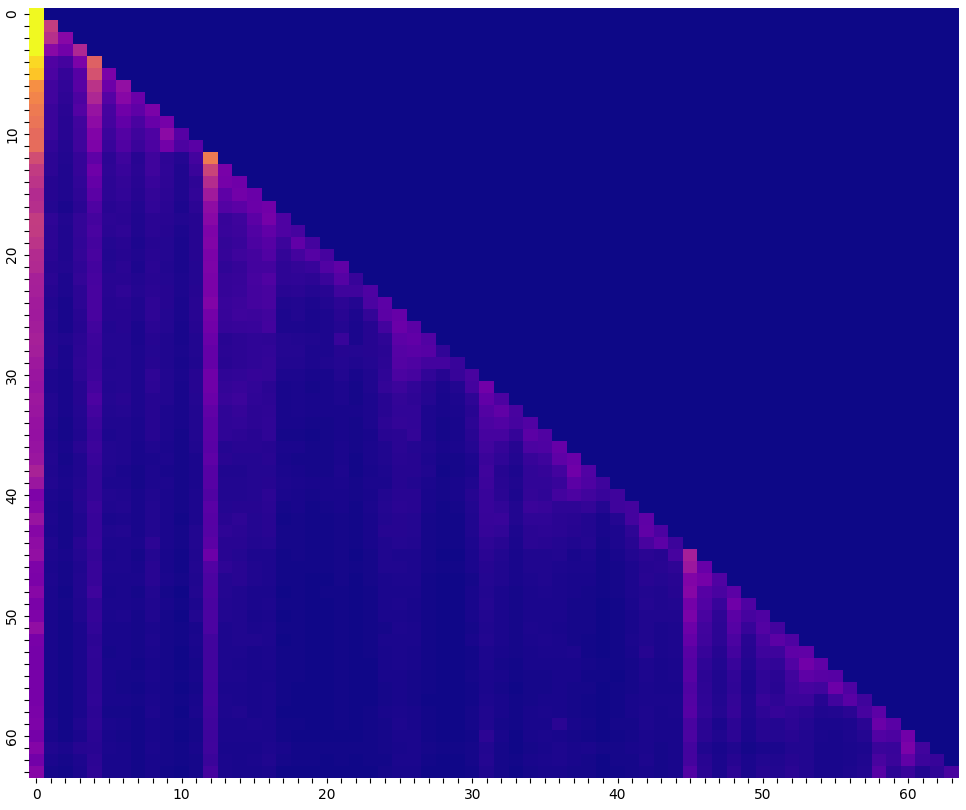}
        \centerline{\small (d) Nexus Inner ($K$-Attn)}
    \end{minipage}
    \caption{Visualization of average attention heatmaps. The x-axis represents Key positions, and the y-axis represents Query positions. Brighter colors indicate higher attention weights. (a) Standard Self-Attention in Pythia. (b) The outer main attention of the Nexus network. (c) The inner recursive attention used to project Queries. (d) The inner recursive attention used to project Keys.}
    \label{fig:attn_maps}
\end{figure*}

\paragraph{Preservation of Causal Structure.}
Comparing the baseline attention (Figure~\ref{fig:attn_maps}a) with the Nexus outer attention (Figure~\ref{fig:attn_maps}b), we observe that both maintain the fundamental structural characteristics of causal language models: a strong diagonal focus (attending to local context) and a prominent vertical band on the left (attending to the beginning-of-sentence token). This similarity is crucial; it indicates that the Higher-Order Attention mechanism preserves the essential capabilities of the Transformer to model sequential order and global context, ensuring stability while enhancing expressiveness.

\paragraph{Role of Inner Attentions.}
The distinct power of our method is revealed in the inner recursive layers (Figure~\ref{fig:attn_maps}c and \ref{fig:attn_maps}d). 
Specifically, the \textbf{Inner $K$-Attention} (Figure~\ref{fig:attn_maps}d) exhibits distinct vertical stripes that are less pronounced in the baseline or Query attention. A vertical stripe in an attention map implies that a specific token at that position is attended to by many subsequent tokens. This suggests that the Inner $K$-Attention acts as a \textit{semantic highlighter}: it identifies and aggregates globally relevant information (e.g., keywords or entities) into the Key representation before the main attention mechanism even computes the final scores. 

\paragraph{Contextualized Projections.}
In a standard Transformer, the Query and Key projections are linear and context-agnostic (i.e., $K_i$ depends only on $x_i$). In contrast, our visualization confirms that in the Nexus network, the vector used as a "Key" for the outer loop is itself a context-aware representation aggregated from previous tokens via the inner loop. The visualization of $\text{Nexus}_Q$ and $\text{Nexus}_K$ demonstrates that the model learns to dynamically adjust these projections based on the sequence context, effectively performing a "pre-reasoning" step that simplifies the task for the final outer attention layer.

\subsection{Retrofitting Standard Transformers for Reasoning}
To demonstrate the versatility and practical applicability of our approach, we investigated whether existing standard Transformer models could be "upgraded" to the Nexus architecture during the Supervised Fine-Tuning (SFT) stage. We utilized the Qwen2.5-Base models (1.5B and 7B) as our starting points.

\paragraph{Setup.}
We employed the Open-R1 framework~\footnote{https://github.com/huggingface/open-r1} to conduct SFT on reasoning-intensive datasets. For the baseline, we directly fine-tuned the standard Qwen2.5 models. For our method, we converted the pre-trained standard attention layers into Higher-Order Attention layers (initializing the recursive components randomly while retaining the pre-trained weights for the outer projections) and then performed SFT under identical settings. We evaluated the models on three challenging reasoning benchmarks: MATH-500, AIME24, and GPQA-Diamond.

\begin{table}[h]
	\centering
	\caption{Performance comparison of Qwen2.5 models fine-tuned with standard architecture vs. retrofitted Nexus architecture on reasoning benchmarks. \textbf{SFT} denotes standard Supervised Fine-Tuning, while \textbf{Nexus-SFT} denotes fine-tuning after converting the architecture. Scores are reported as accuracy.}
	\label{tab:sft_reasoning}
	\resizebox{\linewidth}{!}{
		\begin{tabular}{ll|ccc|c}
			\toprule
			\textbf{Base Model} & \textbf{Method} & \textbf{MATH-500} & \textbf{AIME24} & \textbf{GPQA-Diamond} & \textbf{Avg.} \\
			\midrule
			\multirow{2}{*}{Qwen2.5-1.5B} 
			& Standard SFT & 0.786 & 0.194 & 0.276 & 0.419 \\
			& \textbf{Nexus-SFT (Ours)} & \textbf{0.801} & 0.194 & \textbf{0.280} & \textbf{0.425} \\
			\midrule
			\multirow{2}{*}{Qwen2.5-7B} 
			& Standard SFT & \textbf{0.921} & 0.452 & 0.401 & 0.591 \\
			& \textbf{Nexus-SFT (Ours)} & \textbf{0.921} & \textbf{0.475} & \textbf{0.407} & \textbf{0.601} \\
			\bottomrule
		\end{tabular}
	}
\end{table}

\paragraph{Results and Analysis.}
The results, summarized in Table~\ref{tab:sft_reasoning}, indicate that retrofitting standard Transformers with Higher-Order Attention yields consistent improvements in reasoning capabilities.
\begin{itemize}
    \item \textbf{On the 1.5B scale}, Nexus-SFT achieves a significant gain on MATH-500 (+1.5\%) and improvements on GPQA, demonstrating that smaller models can benefit from the enhanced expressiveness of higher-order interactions to solve math problems.
    \item \textbf{On the 7B scale}, the improvements are particularly notable in the AIME24 benchmark (+2.3\%), which consists of challenging mathematics competitions. This suggests that as the model scale increases, the Nexus architecture effectively leverages the recursive mechanism to handle complex, multi-step reasoning chains required for competition-level mathematics.
\end{itemize}
These findings suggest that the Nexus architecture is not limited to pre-training from scratch but can serve as an effective "architectural upgrade" for existing models, unlocking further potential in complex reasoning tasks with minimal adaptation cost.

\section{Conclusion}

Transformers have achieved significant success across various domains, primarily due to their self-attention mechanisms that efficiently capture long-range dependencies. However, traditional first-order attention methods encounter limitations when addressing complex tasks involving intricate, multi-token relationships. In this paper, we introduce the Higher-Order Attention Network (Nexus network), which employs a novel recursive higher-order attention mechanism to process query and key vectors iteratively, thereby constructing a multi-level attention framework. This approach enhances the model's ability to capture hierarchical dependencies and complex interactions within a single attention layer, leading to improved scalability and performance. Theoretically, we demonstrate that our higher-order mechanism offers greater expressiveness compared to existing methods. Empirical evaluations on multiple benchmark datasets confirm that the Nexus network consistently outperforms standard Transformers. Moreover, we demonstrated that our architecture allows for the effective upcycling of existing pre-trained models. By retrofitting standard Transformers with Higher-Order Attention during the fine-tuning stage, we achieved significant gains in complex mathematical reasoning benchmarks. These advancements position the Nexus network not only as a robust foundation for new models but also as a practical enhancement for the current generation of Large Language Models. Future work will focus on optimizing the computational efficiency of higher-order attentions and exploring their applicability in broader domains, including vision and multimodal learning.

\bibliography{example_paper}
\bibliographystyle{icml2025}

%%%%%%%%%%%%%%%%%%%%%%%%%%%%%%%%%%%%%%%%%%%%%%%%%%%%%%%%%%%%%%%%%%%%%%%%%%%%%%%
%%%%%%%%%%%%%%%%%%%%%%%%%%%%%%%%%%%%%%%%%%%%%%%%%%%%%%%%%%%%%%%%%%%%%%%%%%%%%%%
% APPENDIX
%%%%%%%%%%%%%%%%%%%%%%%%%%%%%%%%%%%%%%%%%%%%%%%%%%%%%%%%%%%%%%%%%%%%%%%%%%%%%%%
%%%%%%%%%%%%%%%%%%%%%%%%%%%%%%%%%%%%%%%%%%%%%%%%%%%%%%%%%%%%%%%%%%%%%%%%%%%%%%%
\newpage
\appendix
\onecolumn
\section{Theorem Proof} \label{ap:A}
\begin{theorem}
	(Linear Bottleneck) (1) Given any $N$ different inputs \(X_m \in \mathbb{R}^{n \times d}, m=1, \dots, N\) and the corresponding 
	target row stochastic matrices \(A_m \in \mathbb{R}^{n \times n}\), as long as \(rank(\log(A_m)) \leq d_k\), there always 
	exists two mappings \(Q, K: \mathbb{R}^{n \times d} \rightarrow \mathbb{R}^{n \times d_k}\) such that
	\begin{equation} \label{eq:14}
		\sm(\frac{Q(X_m)K(X_m)^\top}{\sqrt{d_k}}) = A_m, \quad m=1, \dots, N.	
	\end{equation}
	(2) If $d<n-1$, there exist \(A_m \in \mathbb{R}^{n \times n}\)  that satisfies \(rank(\log(A_m))=1\) but \eqref{eq:14} still does not hold for all 
	linear transformations $Q(X)=XW_q, K(X)=XW_k$.
	
\end{theorem}
\begin{proof}
	(1) Let \(P_m:=\sqrt{d_k}\log(A_m)\), then we have
	\begin{align}
		(\sm(\frac{P_m}{\sqrt{d_k}}))_{ij}&= (\sm(\log(A_m)))_{ij} \\
		&= \exp(\log((A_m)_{ij}))/\sum_{k=1}^{n}\exp(\log((A_m)_{ik})) \\
		&= A_{ij}/  \sum_{k=1}^{n} A_{ik} \\
		&= A_{ij},
	\end{align}
	that is, 
	\begin{equation}
		\sm(\frac{P_m}{\sqrt{d_k}})=A_m
	\end{equation}
	According to the SVD decomposition, there exist matrices 
	\(U_m \in \mathbb{R}^{n \times d_k}, \Sigma_m\in \mathbb{R}^{d_k \times d_k}, V_m \in \mathbb{R}^{n \times d_k}\) such that \(P_m = U_m\Sigma_mV_m^\top\).
	We only need to set $Q(X_m) := U_m\Sigma_m, K(X_m):=V_m$ and then
	\begin{align}
		\sm(\frac{Q(X_m)K(X_m)^\top}{\sqrt{d_k}}) =& \sm(\frac{U_m\Sigma_mV_m^\top}{\sqrt{d_k}}) \\
		=& \sm(\frac{P_m}{\sqrt{d_k}}) \\ 
		=& A_m
	\end{align}
	for \(m=1,\dots N\). \\
	(2) We only need to consider the setting of $N=1$, so we omit  the corner mark $m$. The condition \(d < n-1\) implies that there 
	exists a vector \(\mathbf{a} \in \mathbb{R}^{1 \times n}\) that is linearly 
	independent of \(\text{Col}(X)\cup\{\mathbf{1}\}\).
	Without loss of generality, let's assume \(\exp(\mathbf{a})^T\mathbf{1} = 1 \). 
	Set 
	\begin{equation}
		A := \exp(\mathbf{1}\mathbf{a}^\top), 
	\end{equation}
	then 
	\begin{equation}
		A\mathbf{1} = \exp(\mathbf{1}\mathbf{a}^\top) = \mathbf{1}\exp(a)^\top\mathbf{1}=\mathbf{1}
	\end{equation}
	and thus $A$ is a row stochastic matrix. 
	Now we prove that the standard attention mechanism with linear transformations even cannot represent this \(\log(A)\) of rank 1.
	Assume there exists $Q = XW_q, K=XW_k$ such that
	\begin{equation} \label{eq:24}
		\sm(\frac{QK^\top}{\sqrt{d_k}}) = A.
	\end{equation}
	Then we have 
	\begin{equation}  \label{eq:25}
		A = D\exp{\frac{QK^\top}{\sqrt{d_k}}}
	\end{equation}
	where $D$ is a $n\times n$ diagonal matrix with
	\begin{equation}
		(D)_{ii} = (\sum_{j=1}^{n}\exp(\frac{(QK^\top)_{ij}}{\sqrt{d_k}}))^{-1}
	\end{equation}
	Looking into the first row of \eqref{eq:25}, we obtain that
	\begin{equation} \label{eq:27}
		\exp{\mathbf{a}^\top} = (D)_{11}\exp(\frac{(Q)_{1}K^T}{\sqrt{d_k}}).
	\end{equation}
	Transposing both sides of \eqref{eq:27} and taking the logarithm, we get
	\begin{equation}\label{eq:28}
		\mathbf{a} = \frac{K(Q)_1^{\top}}{\sqrt{d_k}} + \log((D)_{11}) \mathbf{1}
		= \frac{XW_k(Q)_1^{\top}}{\sqrt{d_k}} + \log((D)_{11}) \mathbf{1}.
	\end{equation}
	Equation \eqref{eq:28} contradicts the statement 
	that $\mathbf{a}$ is linearly independent of 
	$\text{Col}(X) \cup \{\mathbf{1}\}$, therefore, 
	there does not exist linear transformations that 
	satisfies \eqref{eq:24} !
\end{proof}
\end{document}